\documentclass[review]{elsarticle}

\usepackage{lineno,hyperref}
\modulolinenumbers[200]
\usepackage{amsthm,amsmath,natbib}
\journal{ArXiv}
\theoremstyle{plain}
\newtheorem{thm}{Theorem}[section]
\theoremstyle{lemma}
\newtheorem{lemma}{Lemma}[section]
\theoremstyle{remark}

\theoremstyle{definition}
\newtheorem{definition}{Definition}[section]

\begin{document}

\begin{frontmatter}

\title{Comment on: Decomposition of structural learning about directed acyclic graphs \cite{xie}}


\author[mymainaddress]{Mohammad Ali Javidian\corref{mycorrespondingauthor}}
\ead{javidian@email.sc.edu}

\author[mymainaddress]{Marco Valtorta\corref{mycorrespondingauthor}}
\cortext[mycorrespondingauthor]{Corresponding author}
\ead{mgv@cse.sc.edu}

\address[mymainaddress]{Department of Computer Science \& Engineering, University of South Carolina, Columbia, SC, 29201, USA.}

\begin{abstract}
We propose an alternative proof concerning necessary and sufficient conditions to split the problem of searching for $d$-separators and building the skeleton of a DAG into small problems for every node of a separation tree $T$. The proof is simpler than the original \cite{xie}. The same proof structure has been used in \cite{jv3} for learning the structure of multivariate regression chain graphs (MVR CGs).
\end{abstract}


\end{frontmatter}

\linenumbers

\section{Introduction}

In this paper we consider directed acyclic graphs (DAGs) and largely use the terminology of \citep{xie}, where the reader can also find further details. For the reader's convenience we just recall the definition of $d$-separation tree here:
\begin{definition}\label{septree}
	A tree $T$ with node set $C$ is said to be a $d$-separation tree for a DAG $G = (V,E)$ if
	\begin{itemize}
		\item $\cup_{C_i\in C}C_i=V$, and
		\item for any separator $S$ in $T$ with $V_1$ and $V_2$, $\langle V_1\setminus S,V_2\setminus S | S\rangle_G$. In other words, $V_1\setminus S$ and $V_2\setminus S$ are $d$-separated by $S$.
	\end{itemize}
\end{definition}

\section{Necessary and sufficient condition for decomposing structural learning of DAGs}
First we give several lemmas (Lemma 2.1-2.4) from \cite[Appendix A]{xie} to be used in the proof of the main theorem in \cite[Theorem 1]{xie}.

\begin{lemma}\label{lem1}
	Let $l$ be a path from $u$ to $v$, and $W$ be the set of all vertices on $l$ ($W$ may or may not contain $u$ and $v$).
	Suppose that (the endpoints of) a path $l$ is (are) $d$-separated by $S$. If $W\subseteq S$, then the path $l$ is $d$-separated by $W$ and by any set containing $W$.
\end{lemma}
\begin{lemma}\label{lem2}
	Let $T$ be a $d$-separation tree for a DAG $G$, and $K$ be a separator of $T$ which separates $T$ into two
	subtrees $T_1$ and $T_2$ with variable sets $V_1$ and $V_2$ respectively. Suppose that $l$ is a path from $u$ to $v$ in $G$ where $u\in V_1\setminus K$ and $v\in V_2\setminus K$. Let $W$ denote the set of all vertices on $l$ ($W$ may or may not contain $u$ and $v$). Then the
	path $l$ is $d$-separated by $W\cap K$ and by any set containing $W\cap K$. 
\end{lemma}
\begin{lemma}\label{lem3}
	Let $u$ and $v$ be two non-adjacent vertices in DAG $G$, and let $l$ be a path from $u$ to $v$. If $l$ is not contained in
	$An(u\cup v)$, then $l$ is $d$-separated by any subset S of $an(u\cup v)$.
\end{lemma}
\begin{lemma}\label{lem4}
	Let $T$ be a $d$-separation tree for a DAG $G$. For any vertex $u$ there exists at least one node of $T$ that contains $u$ and $pa(u)$.
\end{lemma}
Now, we prove a new lemma.
\begin{lemma}\label{lem5}
	Let $T$ be a $d$-separation tree for a DAG $G$ and $C$ a node of $T$. Let $u$ and $v$ be two vertices in $C$ which
	are non-adjacent in $G$, then there exists a node $C'$ of $T$ containing $u, v$ and a set $S$ such that $S$ $d$-separates $u$ and $v$ in $G$.
\end{lemma}
\begin{proof}
	Without loss of generality, we can suppose that $v$ is not a descendant of $u$ in $G$, i.e., $v\in nd(u)$. By Lemma \ref{lem4}, there is a
	node $C_1$ of $T$ that contains $u$ and $pa(u)$. If $v\in C_1$, then $S$ defined as the parents of $u$ $d$-separates $u$ from $v$.
	\begin{figure}[h]
		\centering
		\includegraphics[scale=.75]{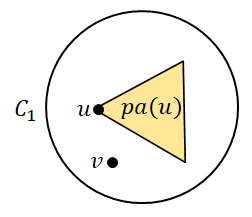}
		 \label{Fig:case1}
	\end{figure}
	
	If $v\not\in C_1$, choose the node $C_2$ which is the closest node in $T$ to the node $C_1$ that contains $u$ and $v$. Consider that there is at least one parent $p$ of $u$ that is not contained
	in $C_2$. Thus there is a separator $K$ connecting $C_2$ toward $C_1$ in $T$ such that $K$ $d$-separates $p$ from all vertices in $C_2\setminus K$. Note that on the path from $C_1$ to $C_2$ in $T$, all
	separators must contain $u$, otherwise they cannot separate $C_1$ from $C_2$. So, we have $u\in K$ but $v\not\in K$ (if $v\in K$, then $C_2$ is not the closest node of $T$ to the node $C_1$). In fact, for every parent $p'$ of $u$ that is contained in $C_1$ but not in $C_2$,  $K$ separates $p'$ from all vertices in $C_2\setminus K$, especially the vertex $v$.
	\begin{figure}[h]
		\centering
		\includegraphics[scale=.7]{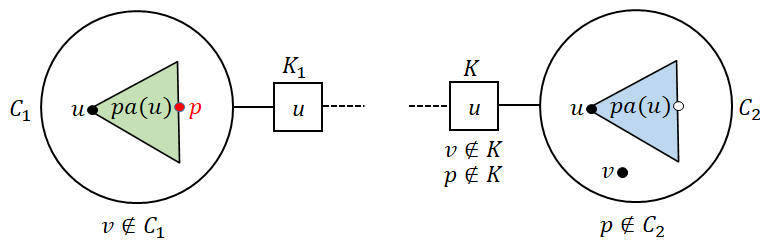}
		\label{Fig:case2}
	\end{figure}

		Define $S=(an(u\cup v)\cap C_2)$. We show that $u$ and $v$ are $d$-separated by $S$, that is, every path, say $l$, between $u$
	and $v$ in $G$ is blocked by $S$.
	
	If $l$ is not contained in $An(u\cup v)$, then we obtain from Lemma \ref{lem3} that $l$ is blocked by $S$.
	
	When $l$ is contained in $An(u\cup v)$, let $x$ be adjacent to $u$ on $l$, that is, $l =
	(u, x, \dots , v)$. The edge between $u$ and $x$ must be oriented as $u\gets x$, that is, $x$ is a parent of $u$. The case $u\to x$ is impossible, because:
	\begin{itemize}
		\item $x\in an(u)$. In this case we have a directed cycle.
		\item $x\in an(v)$. In this case $u$ is an ancestor of $v$ i.e., $v$ is a descendant of $u$ in $G$.
	\end{itemize} 
If $x$ is contained in $C_2$, then $l$ is blocked by $x$ which is contained in $S$. If the parent $x$ of $u$ is not contained in $C_2$, as shown above, we have that $x$ and $v$ are $d$-separated by $K$. By
Lemma \ref{lem2}, we can obtain that the sub-path $l'$ from $x$ to $v$ can be $d$-separated by $W\cap K$ where $W$ denotes the set of
all vertices between $x$ and $v$ (not containing $x$ and $v$) on $l'$. Since $S\supseteq (W\cap K)$, we obtain from Lemma \ref{lem2} that $S$ also blocks $l'$. Hence the path $l$ is blocked by $S$.
\end{proof}
\begin{thm}\label{thm1}
	Let $T$ be a $d$-separation tree for a DAG $G$. Vertices $u$ and $v$ are $d$-separated in $G$ if and
	only if (i) $u$ and $v$ are not contained together in any node $C$ of $T$ or (ii) there exists a node $C$ that contains both $u$
	and $v$ such that a subset $S$ of $C$ $d$-separates $u$ and $v$.
\end{thm}
\begin{proof}[Proof of Theorem \ref{thm1}] 
	\noindent ($\Rightarrow$) Assume that $u$ and $v$ are $d$-separated by $S\subseteq V$ in $G$. If condition (i) is the case, nothing remains to prove. Otherwise, Lemma \ref{lem5} implies condition (ii).
	
	\noindent ($\Leftarrow$) Assume condition (i), i.e., that $u$ and $v$ are not contained together in any node $C$ of $T$. Also, assume that $C_1$ and $C_2$ are two nodes of $T$ that contain $u$ and $v$, respectively. Consider that $C_1'$ is the most distant node from $C_1$, between $C_1$ and $C_2$, that contains $u$ and $C_2'$ is the most distant node from $C_2$, between $C_1$ and $C_2$, that contains $v$. Note that it is possible that $C_1'=C_1$ or $C_2'=C_2$. By the assumption, $C_1'\ne C_2'$. Any separator between $C_1'$ and $C_2'$ satisfies the assumptions of Lemma \ref{lem2}, because it does not contain $u$ or $v$ (otherwise we have a contradiction with the way that we chose $C_1'$ or $C_2'$, or with condition (i)). The sufficiency of condition (i) is given by Lemma \ref{lem2}.
	
	The sufficiency of
	conditions (ii) is trivial by the definition of $d$-separation.
\end{proof}

\section*{Acknowledgements}
This work has been partially supported by Office of Naval Research grant ONR N00014-17-1-2842.  This research is based upon work supported in part by the Office of the Director of National
Intelligence (ODNI), Intelligence Advanced Research Projects Activity (IARPA), award/contract
number 2017-16112300009. The views and conclusions contained therein are those of the authors
and should not be interpreted as necessarily representing the official policies, either expressed or implied,
of ODNI, IARPA, or the U.S. Government. The U.S. Government is authorized to reproduce
and distribute reprints for governmental purposes, notwithstanding annotation therein.

\section*{References}

\bibliography{mybibfile}

\end{document}